\documentclass{article}

\usepackage[dblblindworkshop, final]{neurips_2026}
\workshoptitle{SLM-Agents: 1st Workshop on Small Language Models for Agentic Systems}

\usepackage[utf8]{inputenc} 
\usepackage[T1]{fontenc}    
\usepackage{hyperref}       
\usepackage{url}            
\usepackage{booktabs}       
\usepackage{amsfonts}       
\usepackage{nicefrac}       
\usepackage{microtype}      
\usepackage{xcolor}         
\usepackage{amsmath}
\usepackage{amssymb}
\usepackage{amsthm}
\newtheorem{theorem}{Theorem}[section]

\newtheorem{proposition}[theorem]{Proposition}
\newtheorem{definition}[theorem]{Definition}

\usepackage[table]{xcolor}
\usepackage{bm}

\usepackage{tikz}
\usetikzlibrary{arrows.meta,positioning,fit,calc,matrix,decorations.pathreplacing}
\usepackage[normalem]{ulem} 

\title{Kronecker Product Attention: Memory-Efficient KV Compression via Structured Factorization}

\author{%
  Bahareh Nikpour$^{1}$ \quad
  Hossein Rajabzadeh$^{2}$\thanks{This work was done while the author was affiliated with Huawei Canada.} \quad
  Maryam Dialameh$^{2}$\footnotemark[1] \\
  \textbf{Vahid Partovi Nia}$^{3}$ \quad
  \textbf{Parsa Omidi}$^{1}$ \\[6pt]
  $^{1}$Huawei Technologies, Toronto, Canada \quad
  $^{2}$University of Waterloo, Waterloo, Canada \\
  $^{3}$Huawei Noah's Ark Lab \quad\\[4pt]
  \texttt{bahar.nikpour@h-partners.com}, \;
  \texttt{\{hossein.rajabzadeh, maryam.dialameh\}@uwaterloo.ca}, \\
  \texttt{\{vahid.partovinia, parsa.omidi\}@huawei.com}
}

\begin{document}

\maketitle

\begin{abstract}
Efficient autoregressive generation is increasingly limited by the memory
footprint of the key-value (KV) cache, which grows linearly with sequence
length, batch size, and model depth. Existing approaches such as Multi-Query Attention (MQA), Grouped-Query Attention (GQA),
Multi-Head Latent Attention (MLA), and Tensor Product Attention (TPA) reduce this cost through different trade-offs, but often
sacrifice expressivity, require additional positional-encoding machinery,
or impose restrictive low-rank structure.
We propose Kronecker Product Attention (KPA), which represents key and value
activations as sums of Kronecker products of compact, token-dependent matrix
factors. KPA strictly generalizes TPA: rank-one vector-outer-product factorization
arises as a degenerate Kronecker split, while non-degenerate matrix factors can
represent richer block-structured patterns. We further show that a structured rotary
embedding operator matched to the Kronecker representation preserves the
relative-position property of RoPE while rotating only the compact factors,
removing the need for a decoupled RoPE branch as in MLA.
Empirically, KPA matches or exceeds TPA, MHA, MQA, GQA, and MLA across
four model scales and nine downstream benchmarks under identical training
conditions, while achieving up to $10\times$ KV cache compression.
The Kronecker head split is a tuneable hyperparameter that allows
practitioners to navigate the compression--accuracy trade-off at design
time without architectural changes.
\end{abstract}

\section{Introduction}
The deployment of large-scale Transformer-based language models is increasingly bottlenecked by inference-time memory rather than parameter count. At the core of this bottleneck is the key-value (KV) cache, which stores projected keys and values across every layer and token so they do not need to be recomputed during decoding. This cache grows linearly with sequence length, batch size, and head dimensionality. 
As context windows extend to hundreds of thousands of tokens and batch sizes increase for throughput, the KV cache can dominate total GPU memory consumption, leaving little headroom for model weights or activations and directly limiting the context length, batch size, and deployment scale that practitioners can afford, motivating a broad line of work on KV compression.

Existing approaches address this along several lines. Token eviction and sparse
attention methods risk discarding relevant tokens \citep{zhang2023h2o, xiao2023streamingllm}. Others propose off-chip storage of key-value states, trading memory pressure for increased I/O latency. At the architectural level, Multi-Query Attention (MQA) \cite{shazeer2019mqa} and Grouped-Query Attention (GQA) \citep{ainslie2023gqa} share keys and values across heads at the cost of
representational diversity. Multi-Head Latent Attention (MLA) \citep{liu2024deepseekv2} avoids this by caching compressed latent representations instead, but requires a decoupled RoPE branch, in which a small set of uncompressed key dimensions carry the rotary embedding and are concatenated to the latent key at each decoding step. Tensor Product Attention (TPA) imposes low-rank CP (Canonical Polyadic) structure directly on KV projections, storing only the
low-dimensional factors\citep{zhang2025tpa}. However, CP decomposition operates on vectors and
cannot jointly encode structure at both the head level and the feature level
within each head. This limitation is important as empirical studies have shown that attention heads specialize into distinct functional roles, with different subsets attending to syntactic, semantic, and positional patterns, suggesting that an effective KV factorization should operate at both levels simultaneously \citep{michel2019sixteen, voita2019analyzing}.

\begin{figure}[t]
  \centering
  \includegraphics[width=0.9\linewidth]{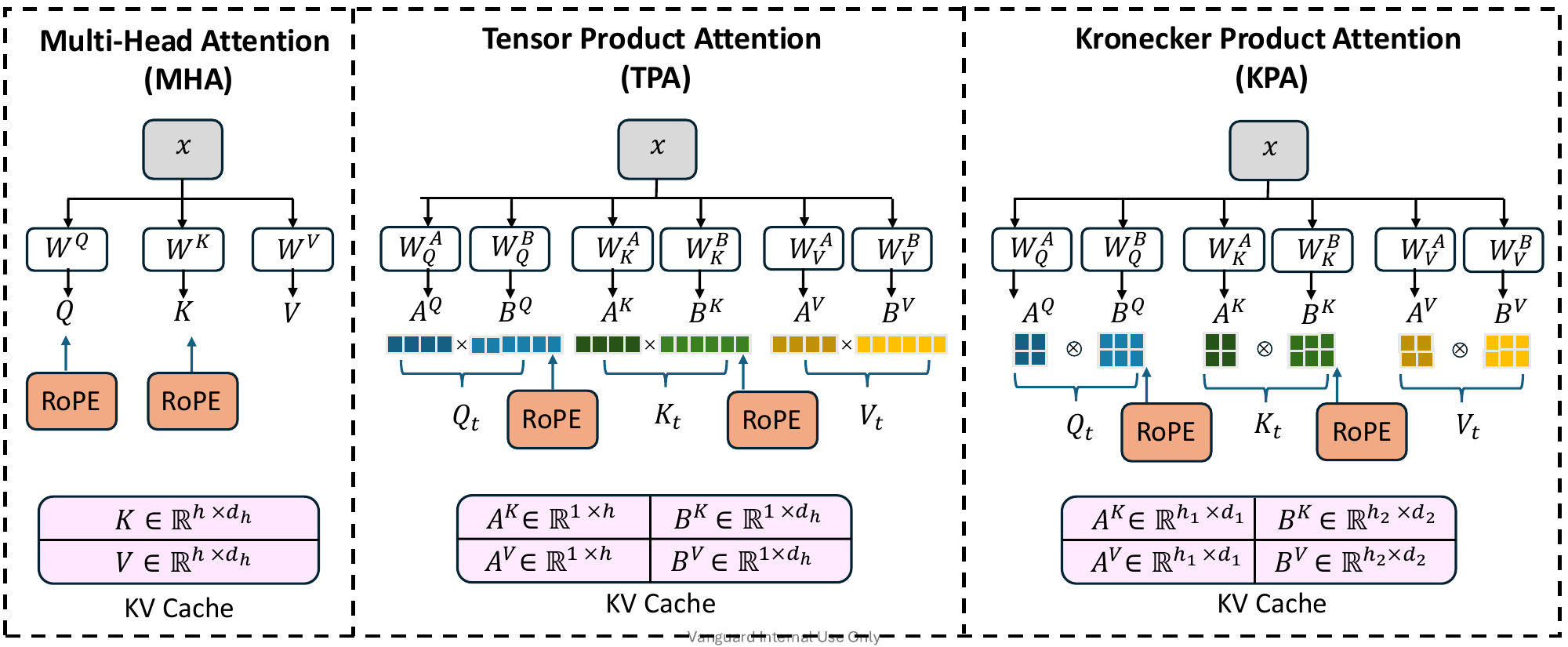}
  \caption{Overview of Multi-Head Attention (MHA), Tensor Product Attention (TPA), and
Kronecker Product Attention (KPA, ours). MHA projects $x$ directly into dense $Q,K,V$.
TPA and KPA instead factor each projection into two components, combined via an outer
product ($\times$) for TPA or a Kronecker product ($\otimes$) for KPA. RoPE is applied
only to $Q$ and $K$ (never $V$), and rotates only the $B$ factor in both methods
(Theorem~\ref{thm:rope} formalizes this for KPA); $A$ is cached unrotated. Since TPA's
$A,B$ are vectors, their outer product is a flat, unstructured matrix; since KPA's
$A,B$ are matrices, their Kronecker product is block-structured with a two-level
head/feature hierarchy (shown by the internal grid lines). Only $K,V$ factors enter
the KV cache; $Q$ is not cached. The token index $t$ is omitted throughout except on
the final $Q_t,K_t,V_t$, for clarity.}
  \label{fig:overview-panels}
\end{figure}

Motivated by this, we use Kronecker product structure to represent per-token
key and value matrices via two smaller contextual matrix factors. Unlike CP
decomposition which operates on vectors, the Kronecker product expands two
small matrices into a larger block-structured matrix with the same shape as
a standard key or value slice, making it naturally suited to multi-head
attention where both the head axis and the within-head feature axis have
product structure. Among structured factorizations, the Kronecker product
is distinguished by combining three properties in a single operation: its
mixed-product identity lets inner products between Kronecker-structured
queries and keys factor into products of smaller inner products; it admits
a Kronecker-separable rotary operator that acts on only one compact factor;
and a rank-1 outer product, which is the basic building block of CP decomposition, is a special case, making Kronecker factorization strictly more general
than CP at the same parameter budget. We note that Kronecker structure has also been explored in attention
mechanisms for encoder models operating on multi-dimensional
inputs \citep{omranpour2024higher}, reducing training-time
computation rather than KV cache size. That setting differs fundamentally
from autoregressive decoding, which is our focus. 

We introduce \textbf{K}ronecker \textbf{P}roduct \textbf{A}ttention \textbf{(KPA)}, a structured attention mechanism that applies this factorization to the query, key, and value projections of each attention head. This yields a family of KV compression schemes parameterized by the factor dimensions, allowing flexible trade-offs between compression ratio, head structure, and representational capacity. We provide a theoretical analysis characterizing the expressivity of KPA as a function of rank and factor dimensions, and prove that a modified frequency schedule we call Kronecker-Separable RoPE integrates seamlessly with the factorized KV representation without materializing the full key or value tensors and without breaking the standard relative-position property.
To evaluate KPA empirically, we train models at four scales including 124M, 353M, 773M, and 1.5B parameters on the FineWeb-Edu corpus. We compare KPA variants against TPA at matched compression ratios, finding that KPA matches or exceeds TPA across all scales. Figure 1 summarizes the contrast between MHA, TPA, and KPA. Our contributions are as follows:
\begin{itemize}

\item We propose Kronecker Product Attention (KPA), the first application
of Kronecker product factorization to \emph{per-token, contextual} key and
value representations for KV cache compression in autoregressive decoding.
Unlike prior Kronecker attention work that decomposes static attention score
matrices for encoder models, KPA operates on
input-dependent factors and directly reduces inference-time memory.

\item We introduce Kronecker-Separable RoPE, which applies rotary position
encoding directly to the compact key factors, preserving the standard
relative-position property without a decoupled RoPE branch as in MLA.

\item We train KPA at four scales (125M--1.5B) on FineWeb-Edu under
identical conditions to TPA, showing KPA matches or exceeds TPA, MHA, MQA,
GQA, and MLA across nine benchmarks while achieving $6.4\times$--$10\times$
KV cache compression.

\end{itemize}

\section{Kronecker Product Attention}
\label{sec:kpa}

The key-value cache in standard multi-head attention stores, for each token $t$,
a pair of matrices $\bm{K}_t, \bm{V}_t \in \mathbb{R}^{h \times d_h}$, whose
entries are dense and unstructured. We argue that this is wasteful: attention
heads are known to specialise into distinct functional roles~\citep{michel2019sixteen,
voita2019analyzing}, suggesting that the \emph{head-level} dimension of $\bm{K}_t$
and $\bm{V}_t$ (i.e., how information is shared or differs across heads) carries
structured, differentiable information. We hypothesise that the \emph{feature-level}
dimension (variation across features within a single head) benefits from an
analogous structure, and design KPA.
The proposed KPA makes this two-level structure explicit by factorising each key and value
matrix as a sum of Kronecker products of two smaller matrices, one operating on
a coarse partition of the head space, the other on the per-head feature space
(formalised in Section~\ref{sec:factorisation}), and caching only these factors.
The result is a compression scheme that is structurally aligned with the
functional organisation of multi-head attention in a way that CP decomposition,
which operates on vectors rather than matrices, cannot achieve.

\subsection{Background and Notation}
\label{sec:background}

\paragraph{Notation.}
Bold uppercase denotes matrices ($\bm{K}_t, \bm{V}_t, \bm{Q}_t, \bm{A},
\bm{B}, \bm{W}^K$); bold lowercase denotes vectors ($\bm{q}_t, \bm{k}_t,
\bm{v}_t, \bm{x}_t$); unbolded lowercase denotes scalars and indices
($h, d_h, d_{\mathrm{model}}, R_K, t, i$); and calligraphic uppercase denotes
higher-order tensors ($\mathcal{K}, \mathcal{V}, \mathcal{Q}$). We write
$\otimes$ for the Kronecker product, $\bm{I}_n$ for the $n \times n$ identity, and $\bm{R}_\theta(t)$ for the RoPE
rotation at position $t$. Subscripts index sequence positions; superscripts
identify projections. We set $[n] := \{1, \ldots, n\}$, and all vectors are
column vectors unless stated otherwise.

\begin{definition}[Kronecker product]
For matrices $\bm{A}\in\mathbb{R}^{m_1\times n_1}$ and
$\bm{B}\in\mathbb{R}^{m_2\times n_2}$, the Kronecker product
$\bm{A}\otimes\bm{B}\in\mathbb{R}^{m_1m_2\times n_1n_2}$ is defined as
\[
\bm{A}\otimes\bm{B}
=
\begin{pmatrix}
a_{11}\bm{B} & \cdots & a_{1n_1}\bm{B}\\
\vdots & \ddots & \vdots\\
a_{m_1 1}\bm{B} & \cdots & a_{m_1 n_1}\bm{B}
\end{pmatrix}.
\]
\end{definition}
The central algebraic property used in KPA is the mixed-product identity
\citep{vanloan1993approximation} which states for conformable matrices
$\bm{M},\bm{N},\bm{P},\bm{Q}$, we can write:
\begin{equation}
(\bm{M}\otimes\bm{N})(\bm{P}\otimes\bm{Q})
= (\bm{M}\bm{P})\otimes(\bm{N}\bm{Q}).
\label{eq:mixed-product}
\end{equation} 

This identity underlies both the factored score computation
(Section~\ref{sec:scores}) and the RoPE compatibility proof
(Section~\ref{sec:rope}).

\subsection{Factorised Key, Value, and Query Representations}
\label{sec:factorisation}

Let $h = h_1 h_2$ and $d_h = d_1 d_2$, where $h_1, h_2, d_1, d_2$ are
hyperparameters that control the two-level structure. For each token $t$ with hidden state
$\bm{x}_t \in \mathbb{R}^{d_{\mathrm{model}}}$, KPA represents the key, value,
and query slices as normalized sums of contextual Kronecker products:
\begin{align}
  \bm{Q}_t &= \frac{1}{R_Q}\sum_{r=1}^{R_Q}
    \bm{A}^Q_r(\bm{x}_t) \otimes \bm{B}^Q_r(\bm{x}_t)
    \;\in \mathbb{R}^{h \times d_h}, \label{eq:Q} \\[4pt]
  \bm{K}_t &= \frac{1}{R_K}\sum_{r=1}^{R_K}
    \bm{A}^K_r(\bm{x}_t) \otimes \bm{B}^K_r(\bm{x}_t)
    \;\in \mathbb{R}^{h \times d_h}, \label{eq:K} \\[4pt]
  \bm{V}_t &= \frac{1}{R_V}\sum_{r=1}^{R_V}
    \bm{A}^V_r(\bm{x}_t) \otimes \bm{B}^V_r(\bm{x}_t)
    \;\in \mathbb{R}^{h \times d_h}, \label{eq:V}
\end{align}
where $\bm{A}^\cdot_r(\bm{x}_t) \in \mathbb{R}^{h_1 \times d_1}$ operates on
the head-partition subspace $\mathbb{R}^{h_1 \times d_1}$ and
$\bm{B}^\cdot_r(\bm{x}_t) \in \mathbb{R}^{h_2 \times d_2}$ operates on the
within-head subspace $\mathbb{R}^{h_2 \times d_2}$. Both factors are
fully contextual meaning they depend on $\bm{x}_t$ through independent learned linear
projections as:
\begin{equation}
  \bm{A}^K_r(\bm{x}_t) = \bm{W}^{A^K}_r\,\bm{x}_t, \qquad
  \bm{B}^K_r(\bm{x}_t) = \bm{W}^{B^K}_r\,\bm{x}_t,
  \label{eq:latent-maps}
\end{equation}
with $\bm{W}^{A^K}_r \in \mathbb{R}^{(h_1 d_1)\times d_{\mathrm{model}}}$ and
$\bm{W}^{B^K}_r \in \mathbb{R}^{(h_2 d_2)\times d_{\mathrm{model}}}$ reshaped
after multiplication. Analogous maps define the query and value factors. In
practice, we merge the rank index into aggregate weight matrices
$\bm{W}^{A^K} \in \mathbb{R}^{(R_K h_1 d_1)\times d_{\mathrm{model}}}$ and
$\bm{W}^{B^K} \in \mathbb{R}^{(R_K h_2 d_2)\times d_{\mathrm{model}}}$,
then reshape to $\mathcal{A}^K(\bm{x}_t) \in \mathbb{R}^{R_K \times h_1
\times d_1}$ and $\mathcal{B}^K(\bm{x}_t) \in \mathbb{R}^{R_K \times h_2
\times d_2}$, enabling a single batched projection per token.

The ranks $R_Q, R_K, R_V$ are independent. Since queries are never cached,
one may freely set $R_Q \gg R_K, R_V$ to enrich query expressivity at no
memory cost. Only the $\bm{K}$ and $\bm{V}$ factors enter the cache.

It should be noted that setting $h_1 = h$, $d_1 = 1$, $h_2 = 1$, $d_2 = d_h$ makes $\bm{A}_r \in \mathbb{R}^{h \times 1}$
and $\bm{B}_r \in \mathbb{R}^{1 \times d_h}$ vectors, so each Kronecker term reduces to
a rank-one outer product, exactly recovering TPA~\citep{zhang2025tpa}. Non-degenerate
Kronecker factors are strictly richer: $\mathrm{rank}(\bm{A}_r \otimes \bm{B}_r) =
\mathrm{rank}(\bm{A}_r) \cdot \mathrm{rank}(\bm{B}_r)$, giving KPA strictly greater
representational capacity at matched parameter budgets.

\subsection{Factored Attention Score Computation}
\label{sec:scores}

Heads in KPA are indexed by pairs $(i_1, i_2)$ with $i_1 \in [h_1]$, $i_2 \in [h_2]$.
Denoting the $i_1$-th row of $\bm{A}^Q_r$ as $(\bm{A}^Q_r)_{i_1} \in \mathbb{R}^{d_1}$,
the mixed-product identity~\eqref{eq:mixed-product} gives an exact factorisation of the
attention score (scaled by $1/\sqrt{d_h}$ before softmax in practice):

\begin{proposition}[Factored Score Computation]
\label{prop:score}
The attention score for head $(i_1, i_2)$ at positions $t$ and $s$ satisfies
\begin{equation}
  \mathrm{score}_{(i_1,i_2)}(t,s)
  = \frac{1}{R_Q R_K}
    \sum_{r=1}^{R_Q}\sum_{r'=1}^{R_K}
    \underbrace{(\bm{A}^Q_r)_{i_1}^\top (\bm{A}^K_{r'})_{i_1}}_{\text{head-level}}
    \underbrace{(\bm{B}^Q_r)_{i_2}^\top (\bm{B}^K_{r'})_{i_2}}_{\text{feature-level}}.
  \label{eq:score}
\end{equation}
\end{proposition}
\begin{proof}See Appendix A.\end{proof}

Every attention score decomposes into a product of a \emph{head-level} and a
\emph{feature-level} inner product, reflecting the two-level structure of KPA.
Proposition~\ref{prop:score} shows that scores can in principle be computed directly
from the compact factors without materializing the full key tensor, analogously to
factored score computation in TPA~\citep{zhang2025tpa}.

\subsection{Kronecker-Separable RoPE}
\label{sec:rope}

A natural concern with any structured KV factorisation is compatibility with
rotary positional embeddings (RoPE)~\citep{su2024rope}. MLA resolves this by
maintaining a small set of non-compressed KV dimensions that carry the rotary
encoding, introducing a decoupled RoPE branch at every decoding step. We show
that KPA requires no such branch.

\begin{definition}[Kronecker-Separable RoPE]
\label{def:ks-rope}
Let $\bm{I}_{d_1}$ be the $d_1\times d_1$ identity and
$\bm{R}^{(2)}_\theta(t)\in\mathbb{R}^{d_2\times d_2}$ be a RoPE rotation matrix
with frequencies calibrated to $d_2$:
$\theta_j = 10000^{-2j/d_2}$, $j=0,\ldots,d_2/2-1$.
The \emph{Kronecker-Separable RoPE} operator on $\mathbb{R}^{d_h}$ is
$\bm{R}_{\mathrm{KPA}}(t) = \bm{I}_{d_1}\otimes \bm{R}^{(2)}_\theta(t)$.
\end{definition}

\begin{theorem}[RoPE Compatibility of KPA]
\label{thm:rope}
Under Kronecker-Separable RoPE, applying
$\bm{I}_{d_1} \otimes \bm{R}^{(2)}_\theta(\cdot)$ to the query or key vector
of head $(i_1, i_2)$ is equivalent to rotating only the $\bm{B}$ factor:
\begin{align}
  \bigl(\bm{I}_{d_1} \otimes \bm{R}^{(2)}_\theta(t)\bigr)
  \bigl((\bm{A}^Q_r)_{i_1} \otimes (\bm{B}^Q_r)_{i_2}\bigr)
  &= (\bm{A}^Q_r)_{i_1} \otimes \bm{R}^{(2)}_\theta(t)(\bm{B}^Q_r)_{i_2}, \\[4pt]
  \bigl(\bm{I}_{d_1} \otimes \bm{R}^{(2)}_\theta(s)\bigr)
  \bigl((\bm{A}^K_{r'})_{i_1} \otimes (\bm{B}^K_{r'})_{i_2}\bigr)
  &= (\bm{A}^K_{r'})_{i_1} \otimes \bm{R}^{(2)}_\theta(s)(\bm{B}^K_{r'})_{i_2}.
\end{align}
Moreover, the attention score depends only on the relative offset $s - t$:
\begin{equation}
  \mathrm{score}_{(i_1,i_2)}(t,s)
  = \frac{1}{R_Q R_K}\sum_{r,r'}
    \Bigl[(\bm{A}^Q_r)_{i_1}^\top (\bm{A}^K_{r'})_{i_1}\Bigr]
    \cdot
    \Bigl[(\bm{B}^Q_r)_{i_2}^\top\,
          \bm{R}^{(2)}_\theta(s-t)\,
          (\bm{B}^K_{r'})_{i_2}\Bigr].
  \label{eq:rope-score}
\end{equation}
\end{theorem}
\begin{proof}See Appendix A.\end{proof}

Theorem~\ref{thm:rope} has two practical consequences: (i) only the $\bm{B}^K$
factors need to be pre-rotated before caching; the $\bm{A}^K$ factors are
position-independent and stored as-is, and (ii) the $\bm{B}$ factor structure
is preserved by the rotation, so positional encoding need not require
reconstructing the full key tensor. Unlike MLA, no decoupled RoPE branch
is required.
\subsection{KV Cache Cost and Compression}
\label{sec:cache}

During autoregressive decoding, KPA caches the key and value factors per token $t$:
\[
  \{\bm{A}^K_r(\bm{x}_t),\,\bm{B}^K_r(\bm{x}_t)\}_{r=1}^{R_K}
  \;\cup\;
  \{\bm{A}^V_r(\bm{x}_t),\,\bm{B}^V_r(\bm{x}_t)\}_{r=1}^{R_V},
\]
where $\bm{A}^K_r \in \mathbb{R}^{h_1 \times d_1}$ and
$\bm{B}^K_r \in \mathbb{R}^{h_2 \times d_2}$ (pre-rotated via
Kronecker-Separable RoPE). The per-token cache cost is
$(R_K+R_V)(h_1d_1+h_2d_2)$, compared with $2hd_h = 2h_1h_2d_1d_2$ for MHA.
The compression factor is:
\begin{equation}
  C_{\mathrm{KPA}}
  = \frac{2h_1h_2d_1d_2}{(R_K+R_V)(h_1d_1+h_2d_2)}.
  \label{eq:compression}
\end{equation}
For fixed $h$ and $d_h$, compression increases with factor asymmetry: a more
unequal split of $h$ or $d_h$ makes $h_1d_1 + h_2d_2$ smaller relative to
$h_1h_2d_1d_2$. 
\section{Experiments}
\label{sec:experiments}

\subsection{Experimental Setup}

We train KPA models at four scales (125M, 353M, 773M, 1.5B) using a Llama-inspired architecture with SwiGLU and RoPE on top of nanoGPT~\citep{karpathy2022nanogpt},
trained on FineWeb-Edu for $\sim$50B tokens with $R_K{=}R_V{=}2$ and $R_Q{=}6$.
We compare against MHA, MQA, GQA, MLA, and TPA baselines from
\citet{zhang2025tpa} under identical conditions, evaluating zero-shot on nine
benchmarks using \texttt{lm-evaluation-harness}~\citep{eval-harness}.
Full training and evaluation details are in Appendix C.

\subsection{Main Results}
\label{sec:main_results}

Tables~\ref{tab:353M} and~\ref{tab:773M} report zero-shot benchmark performance
at 353M and 773M scales; 1.5B and 125M results appear in Appendix D.

At 353M, all three KPA configurations outperform TPA. KPA~$h{=}(8{\times}2)$
achieves the best average at $6.4\times$ compression, matching TPA's ratio
while exceeding its accuracy by 1.49pp, with the largest gains on ARC-Easy
and SciQ. KPA~$h{=}(4{\times}4)$ outperforms TPA at strictly higher
compression ($8\times$), demonstrating that increasing compression need not
sacrifice accuracy. At 773M, KPA~$h{=}(10{\times}2)$ outperforms TPA, MLA, and all other
baselines at $6.7\times$ compression, with particularly strong gains on
ARC-Easy, BoolQ, and SciQ. Notably, KPA surpasses MLA despite requiring no
decoupled RoPE branch. KPA~$h{=}(4{\times}5)$ achieves better accuracy
than TPA at higher compression ($8.9\times$).

In both cases, the best-performing KPA variant uses an asymmetric split
with larger $h_1$, consistent with the ablation findings in
Section~\ref{sec:ablation}. KPA improves monotonically across scales and
matches or exceeds TPA at comparable or higher compression ratios at every scale.
\subsection{Ablation on Head-Split Asymmetry}
\label{sec:ablation}

We ablate the effect of the Kronecker head split $h_1 \times h_2 = h$,
holding $d_1 = d_2 = 8$ and $R_K = R_V = 2$ fixed. At 353M we compare
symmetric $h{=}(4{\times}4)$, asymmetric $h{=}(2{\times}8)$, and reversed
asymmetric $h{=}(8{\times}2)$; at 773M we compare $h{=}(4{\times}5)$,
$h{=}(2{\times}10)$, and $h{=}(10{\times}2)$. Tables~\ref{tab:353M} and~\ref{tab:773M} reveal a consistent pattern across
scales. Asymmetric splits with larger $h_1$ outperform both symmetric KPA
and TPA at both scales. At 353M, $h{=}(8{\times}2)$ achieves the best
average, and at 773M, $h{=}(10{\times}2)$ outperforms all baselines.
We attribute this to the $\bm{A}$-factor's capacity to encode joint
interactions across $h_1$ heads simultaneously in feature space. A larger
$h_1$ expands this capacity, whereas the $\bm{B}$-factor refines
within-head feature structure independently, which appears less critical.
This is distinct from MHA, where head projections are fully independent
by construction.
Symmetric splits trade some accuracy for higher compression
($8\times$ to $8.9\times$), offering a favourable operating point when
memory is the primary constraint. The head split $h_1 \times h_2$ is a
tuneable hyperparameter that allows practitioners to navigate the
compression--accuracy trade-off at design time, with no other architectural
changes required.

\begin{table*}[t]
\centering
\caption{Zero-shot evaluation on nine benchmarks (FineWeb-Edu, $\sim$50B tokens),
353M-parameter models. Best score per column in \textbf{bold}.
HellaSw.\ = HellaSwag, W.G.\ = WinoGrande.}

\label{tab:353M}
\setlength{\tabcolsep}{5pt}
\scriptsize
\begin{tabular}{lcccccccccr}
\toprule
\textbf{Method} & \textbf{ARC-E} & \textbf{ARC-C} & \textbf{BoolQ} & \textbf{HellaSw.} & \textbf{OBQA} & \textbf{PIQA} & \textbf{W.G.} & \textbf{MMLU} & \textbf{SciQ} & \textbf{Avg.} \\
\midrule
MHA   & 59.51 & 29.52 & 59.60 & 45.68 & 34.20 & 68.82 & 53.43 & 23.33 & 76.90 & 50.11 \\
MQA   & 57.62 & 31.91 & 59.45 & 45.69 & 35.40 & 69.31 & 53.51 & 26.47 & 74.60 & 50.44 \\
GQA   & 58.67 & 31.48 & 58.29 & 45.45 & 35.20 & 68.50 & 54.46 & 24.58 & 76.50 & 50.35 \\
MLA   & 56.65 & 29.52 & 57.83 & 46.05 & 34.60 & 69.42 & 52.80 & 24.62 & 79.70 & 50.13 \\
TPA   & 58.38 & 31.57 & 59.39 & \textbf{46.83} & \textbf{37.00} & 70.02 & 54.06 & 25.52 & 79.90 & 51.41 \\
\midrule
KPA $h{=}(4{\times}4)$ & 64.23 & 30.72 & 59.66 & 44.09 & 34.40 & 70.08 & 54.38 & \textbf{26.79} & \textbf{85.30} & 52.18 \\
KPA $h{=}(2{\times}8)$ & 64.15 & 31.48 & 57.21 & 45.93 & 36.20 & 70.18 & \textbf{54.92}& 26.31 & 83.90 & 52.25 \\
KPA $h{=}(8{\times}2)$  & \textbf{66.16} & \textbf{33.02} & \textbf{60.37} & 45.11 & 35.80 & \textbf{71.29} & 53.12 & \textbf{26.79} & 84.40 & \textbf{52.90} \\
\bottomrule
\end{tabular}
\end{table*}

\begin{table*}[t]
\centering
\caption{Zero-shot evaluation on nine benchmarks (FineWeb-Edu, $\sim$50B tokens),
773M-parameter models. Best score per column in \textbf{bold}.
HellaSw.\ = HellaSwag, W.G.\ = WinoGrande.}
\label{tab:773M}
\setlength{\tabcolsep}{5pt}
\scriptsize
\begin{tabular}{lcccccccccr}
\toprule
\textbf{Method} & \textbf{ARC-E} & \textbf{ARC-C} & \textbf{BoolQ} & \textbf{HellaSw.} & \textbf{OBQA} & \textbf{PIQA} & \textbf{W.G.} & \textbf{MMLU} & \textbf{SciQ} & \textbf{Avg.} \\
\midrule
MHA   & 59.93 & 33.62 & 61.93 & 50.63 & 36.00 & 71.06 & 55.41 & 22.87 & 81.20 & 52.52 \\
MQA   & 60.73 & 33.62 & 57.34 & 50.09 & 37.00 & 69.97 & 55.49 & \textbf{25.30} & 79.60 & 52.13 \\
GQA   & 61.66 & 34.30 & 58.72 & 49.85 & 38.40 & 71.16 & 53.75 & 25.23 & 77.60 & 52.30 \\
MLA   & 63.55 & 32.85 & 60.95 & \textbf{51.72} & 38.80 & 70.51 & 55.01 & 24.55 & 81.90 & 53.32 \\
TPA   & 63.22 & 35.58 & 60.03 & 51.26 & 36.80 & 71.44 & 55.56 & 24.77 & 79.60 & 53.14 \\
\midrule
KPA $h{=}(4{\times}5)$   & 68.56 & 35.92 & 60.09 & 50.28 & 36.80 & \textbf{72.85} & 55.33 & 25.26 & 85.90 & 54.55 \\
KPA $h{=}(2{\times}10)$  & 68.31 & 34.98 & 55.32 & 50.91 & 38.00 & 71.06 & 55.64 & 24.66 & 85.30 & 53.80 \\
KPA $h{=}(10{\times}2)$  & \textbf{68.89} & \textbf{36.01} & \textbf{62.03} & 50.38 & \textbf{39.40} & \textbf{72.85} & \textbf{55.98} & 24.52 & \textbf{86.00} & \textbf{55.12} \\
\bottomrule
\end{tabular}
\end{table*}
\vspace{-4mm}
\section{Conclusion}
\label{sec:conclusion}

We introduced Kronecker Product Attention (KPA), a structured factorization
of key and value projections that strictly generalizes TPA and integrates
natively with RoPE without requiring a decoupled branch as in MLA. KPA
matches or outperforms TPA, MHA, MQA, GQA, and MLA across all evaluated
scales under identical training conditions, while achieving substantial
KV cache compression. The Kronecker head split is a tuneable hyperparameter
that allows practitioners to navigate the compression--accuracy trade-off
at design time without architectural changes. By reducing the KV cache
footprint, KPA directly addresses the memory-bandwidth bottleneck that
limits language models deployment on resource-constrained devices such as smartphones,
robotics platforms, and embedded systems,  making it a natural fit for
agentic systems that must plan and act under strict memory budgets without
relying on cloud infrastructure. The same compression benefits apply
during fine-tuning for tool use and multi-step reasoning, where longer
context windows are critical but memory is scarce. Developing an efficient
factored decoding kernel and evaluating KPA on longer contexts where
cache compression is most practically impactful are natural next steps.

\bibliographystyle{plainnat}
\bibliography{kpa_references}

\newpage
\appendix

\newpage

\appendix

\section{Proofs}
\label{app:proofs}

\subsection{Proof of Proposition~\ref{prop:score}}

\begin{proof}
For head $(i_1, i_2)$, the per-head query and key vectors are
\[
  \bm{q}_{t,(i_1,i_2)}
  =
  \frac{1}{R_Q}
  \sum_r
  (\bm{A}^Q_r)_{i_1}
  \otimes
  (\bm{B}^Q_r)_{i_2},
  \qquad
  \bm{k}_{s,(i_1,i_2)}
  =
  \frac{1}{R_K}
  \sum_{r'}
  (\bm{A}^K_{r'})_{i_1}
  \otimes
  (\bm{B}^K_{r'})_{i_2}.
\]

Taking the dot product and applying the mixed-product identity
\[
(\bm{a}\otimes\bm{b})^\top
(\bm{c}\otimes\bm{d})
=
(\bm{a}^\top\bm{c})
(\bm{b}^\top\bm{d}),
\]
we obtain
\[
  \bm{q}_{t,(i_1,i_2)}^\top
  \bm{k}_{s,(i_1,i_2)}
  =
  \frac{1}{R_Q R_K}
  \sum_{r,r'}
  \bigl[
    (\bm{A}^Q_r)_{i_1}^\top
    (\bm{A}^K_{r'})_{i_1}
  \bigr]
  \bigl[
    (\bm{B}^Q_r)_{i_2}^\top
    (\bm{B}^K_{r'})_{i_2}
  \bigr].
  \qedhere
\]
\end{proof}

\subsection{Proof of Theorem~\ref{thm:rope}}

\begin{proof}
Applying
$\bm{I}_{d_1}\otimes\bm{R}^{(2)}_\theta(t)$
to
$(\bm{A}^Q_r)_{i_1}\otimes(\bm{B}^Q_r)_{i_2}$
and using the mixed-product identity gives
\[
  (\bm{I}_{d_1}\otimes\bm{R}^{(2)}_\theta(t))
  \bigl(
    (\bm{A}^Q_r)_{i_1}
    \otimes
    (\bm{B}^Q_r)_{i_2}
  \bigr)
  =
  (\bm{A}^Q_r)_{i_1}
  \otimes
  \bm{R}^{(2)}_\theta(t)
  (\bm{B}^Q_r)_{i_2}.
\]

Summing over $r$ and averaging by $1/R_Q$ preserves the factored form.

The relative-position identity then follows from
\[
  \bm{R}^{(2)}_\theta(t)^\top
  \bm{R}^{(2)}_\theta(s)
  =
  \bm{R}^{(2)}_\theta(s-t),
\]
which is the standard RoPE relative-position property restricted to the
$d_2$-dimensional subspace.
\qedhere
\end{proof}

\section{KV Cache Compression}
\label{app:cache}

During autoregressive decoding, KPA caches the factor matrices per past token
at a cost of
$(R_K+R_V)(h_1 d_1+h_2 d_2)$
per token, compared with
$2h_1 h_2 d_1 d_2$
for MHA. The compression factor is
\begin{equation}
  C_{\mathrm{KPA}}
  =
  \frac{2\,h_1 h_2 d_1 d_2}
  {(R_K+R_V)(h_1 d_1+h_2 d_2)}.
  \label{eq:compression-app}
\end{equation}

Table~\ref{tab:compression} reports $C_{\mathrm{KPA}}$ for all evaluated
configurations with $R_K=R_V=2$ and $d_1=d_2=8$.

\begin{table}[h]
\centering
\caption{KV cache compression factors for evaluated KPA configurations.}
\label{tab:compression}
\setlength{\tabcolsep}{6pt}
\small
\begin{tabular}{llr}
\toprule
\textbf{Scale} & \textbf{Config} & $C_{\mathrm{KPA}}$ \\
\midrule
353M & $h{=}(4{\times}4)$   & $8.0\times$ \\
353M & $h{=}(2{\times}8)$   & $6.4\times$ \\
353M & $h{=}(8{\times}2)$   & $6.4\times$ \\
773M & $h{=}(4{\times}5)$   & $8.9\times$ \\
773M & $h{=}(2{\times}10)$  & $6.7\times$ \\
773M & $h{=}(10{\times}2)$  & $6.7\times$ \\
1.5B & $h{=}(5{\times}5)$   & $10.0\times$ \\
\midrule
TPA  & all scales           & $6.4\times$ \\
\bottomrule
\end{tabular}
\end{table}

\section{Training and Evaluation Details}
\label{app:training}

\paragraph{Architecture and data.}
All models use the nanoGPT codebase~\citep{karpathy2022nanogpt} with SwiGLU
activations~\citep{shazeer2020glu} and RoPE embeddings~\citep{su2024rope}.
We fix $d_1 = d_2 = 8$ and vary $h_1, h_2$ across configurations to study
head-split asymmetry. Models are trained from scratch on
FineWeb-Edu-100B~\citep{penedo2024fineweb} for 50{,}000 iterations with a
batch size yielding approximately 1M tokens per iteration, totalling
$\sim$50B tokens.

\paragraph{Optimisation.}
We use AdamW with peak learning rate $6 \times 10^{-4}$ ($3 \times 10^{-4}$
at 773M and 1.5B scales), cosine decay to one-tenth of the peak, 2{,}000
warmup steps, and weight decay 0.1. All models train in bfloat16 on
NVIDIA H100 GPUs.

\paragraph{Evaluation.}
We evaluate zero-shot performance on nine benchmarks using
\texttt{lm-evaluation-harness}~\citep{eval-harness}: ARC-Easy, ARC-Challenge,
BoolQ, HellaSwag, OpenbookQA, PIQA, WinoGrande, MMLU, and SciQ. We report
accuracy-normalized scores for ARC-E, ARC-C, HellaSwag, OBQA, PIQA, and SciQ,
and standard accuracy for BoolQ, WinoGrande, and MMLU. The primary scalar
metric is the unweighted average across all nine tasks.
\section{Additional Experimental Results}
\label{app:additional}

\subsection{1.5B Scale}
\label{app:1p5B}

Since $h{=}25$ admits only the symmetric Kronecker split $h_1{=}h_2{=}5$,
a single KPA configuration is evaluated at the 1.5B scale. KPA outperforms
TPA and all other baselines on eight of nine benchmarks, with the only
exception being MMLU where MHA achieves a slightly higher score. The gains
are most pronounced on ARC-Easy, SciQ, and BoolQ, consistent with the
pattern observed at 353M and 773M scales. This confirms that KPA scales
favourably with model size even when the head split is constrained to a
single symmetric factorization.

\subsection{125M Scale}
\label{app:125M}

At 125M scale, all KPA configurations outperform TPA (46.21\%). Consistent
with the pattern at larger scales, asymmetric splits with larger $h_1$
perform better: $h{=}(6{\times}2)$ achieves the best average, followed by
$h{=}(4{\times}3)$, $h{=}(3{\times}4)$, and $h{=}(2{\times}6)$. The gains
over TPA are smaller than at 353M and 773M, which is expected as smaller
models have less capacity to benefit from richer factorization. Nevertheless,
KPA consistently outperforms all baselines including MHA, MQA, GQA, MLA,
and TPA across all four configurations.

\begin{table}[h]
\centering
\caption{Zero-shot evaluation of 1.5B-parameter models
(FineWeb-Edu, $\sim$50B tokens). Baselines from
\citet{zhang2025tpa}. Best score in each column is \textbf{bold}.}
\label{tab:1p5B}
\setlength{\tabcolsep}{4pt}
\scriptsize
\begin{tabular}{lcccccccccr}
\toprule
\textbf{Method} & \textbf{ARC-E} & \textbf{ARC-C} & \textbf{BoolQ} & \textbf{HellaSw.} & \textbf{OBQA} & \textbf{PIQA} & \textbf{W.G.} & \textbf{MMLU} & \textbf{SciQ} & \textbf{Avg.} \\
\midrule
MHA  & 64.81 & 35.41 & 61.90 & 54.32 & 37.20 & 72.74 & 55.80 & \textbf{25.44} & 82.80 & 54.49 \\
MQA  & 64.10 & 36.01 & 62.26 & 54.38 & 39.00 & 72.58 & 56.43 & 23.70 & 81.90 & 54.48 \\
GQA  & 63.68 & 35.92 & 60.46 & 54.17 & 38.40 & 73.56 & 56.27 & 24.77 & 81.70 & 54.33 \\
MLA  & 64.14 & 35.92 & 60.12 & 53.60 & 39.20 & 72.25 & 55.17 & 24.71 & 81.60 & 54.08 \\
TPA  & 66.71 & 36.52 & 61.38 & 54.03 & 40.40 & 72.52 & 56.83 & 24.49 & 82.20 & 55.01 \\
\midrule
KPA $h{=}(5{\times}5)$ & \textbf{70.62} & \textbf{38.74} & \textbf{63.26} & \textbf{54.60} & \textbf{41.00} & \textbf{73.72} & \textbf{57.20} & 24.38 & \textbf{85.50} & \textbf{56.56} \\
\bottomrule
\end{tabular}
\end{table}

\begin{table}[h]
\centering
\caption{Zero-shot evaluation of 125M-parameter models
(FineWeb-Edu, $\sim$50B tokens). Baselines from
\citet{zhang2025tpa}. Best score in each column is \textbf{bold}.}
\label{tab:125M}
\setlength{\tabcolsep}{4pt}
\scriptsize
\begin{tabular}{lcccccccccr}
\toprule
\textbf{Method} & \textbf{ARC-E} & \textbf{ARC-C} & \textbf{BoolQ} & \textbf{HellaSw.} & \textbf{OBQA} & \textbf{PIQA} & \textbf{W.G.} & \textbf{MMLU} & \textbf{SciQ} & \textbf{Avg.} \\
\midrule
MHA  & 50.63 & 26.96 & \textbf{59.39} & 36.18 & 32.00 & 64.96 & \textbf{51.85} & 23.40 & 70.30 & 46.19 \\
MQA  & 49.62 & 25.34 & 55.72 & 35.94 & 31.40 & 64.85 & 51.30 & 23.37 & 68.70 & 45.14 \\
GQA  & 48.70 & 25.68 & 56.15 & 35.58 & 31.40 & 64.91 & 51.62 & 23.12 & 68.20 & 45.04 \\
MLA  & 50.21 & 26.71 & 58.01 & 36.25 & \textbf{32.80} & 64.69 & 50.59 & \textbf{24.67} & 71.90 & 46.20 \\
TPA  & 51.26 & 27.39 & 57.00 & \textbf{36.68} & \textbf{32.80} & 64.47 & 49.72 & 24.61 & 72.00 & 46.21 \\
\midrule
KPA $h{=}(3{\times}4)$ & 52.10 & 27.30 & 57.58 & 35.90 & 32.20 & 65.10 & 50.80 & 24.50 & 73.80 & 46.59 \\
KPA $h{=}(4{\times}3)$ & 52.89 & 27.52 & 57.84 & 36.06 & 32.40 & 65.30 & 51.10 & 24.55 & 74.20 & 46.87 \\
KPA $h{=}(2{\times}6)$ & 52.30 & 27.20 & 57.22 & 35.78 & 32.20 & 65.09 & 50.90 & 24.40 & 73.50 & 46.51 \\
KPA $h{=}(6{\times}2)$ & \textbf{53.51} & \textbf{27.66} & 57.90 & 35.70 & 32.60 & \textbf{65.60} & 51.20 & 24.59 & \textbf{74.80} & \textbf{47.06} \\
\bottomrule
\end{tabular}
\end{table}

\section{Related Work}
\label{app:related}

\subsection{Head sharing}

MQA \citep{shazeer2019mqa} reduces the KV cache by sharing a single key
and value head across all query heads, and GQA \citep{ainslie2023gqa}
generalizes this by grouping query heads to share a smaller set of key-value
heads, recovering a spectrum between MQA and full MHA.
MLA \citep{liu2024deepseekv2} instead projects keys and values into a
low-dimensional latent space before caching, achieving stronger compression,
but requires a decoupled RoPE branch to handle positional encoding.
KPA avoids head sharing entirely and eliminates the decoupled RoPE branch
by exploiting the Kronecker structure of its projections.

\subsection{Low-rank and structured attention}

Linformer \citep{wang2020linformer} approximates the attention matrix via
low-rank projection of keys and values along the sequence dimension, while
Performer \citep{choromanski2021performer} approximates the softmax kernel
with random features. These methods target the quadratic complexity of
attention rather than KV cache size and do not apply to autoregressive
decoding. Monarch \citep{dao2022monarch} and Tensorizing-NN
\citep{novikov2015tensorizing} apply structured matrix factorizations to
weight matrices for parameter efficiency, but do not specifically target
KV cache compression. TPA \citep{zhang2025tpa} is the closest prior work,
applying CP decomposition to KV projections; KPA strictly generalizes TPA
by replacing rank-1 vector outer products with Kronecker matrix factors.

\subsection{KV compression, eviction, and quantization}

H2O \citep{zhang2023h2o} and StreamingLLM \citep{xiao2024streamingllm}
reduce cache size by evicting tokens deemed less important based on
attention scores, but risk discarding tokens that become relevant later.
KIVI \citep{liu2024kivi} and KVQuant \citep{hooper2024kvquant}
compress the cache through low-bit quantization, trading precision for
memory. PagedAttention \citep{kwon2023paged} improves memory management
through virtual paging but does not reduce the total cache footprint.
KPA takes a complementary architectural approach, reducing the cache at
the representation level before any eviction or quantization is applied.

\subsection{Positional encoding}

RoPE \citep{su2024rope} encodes relative position by rotating query and
key vectors and has become the dominant positional scheme in modern
language models. ALiBi \citep{press2022alibi} instead adds
position-dependent biases to attention logits, avoiding explicit
positional embeddings. Extensions such as YaRN \citep{peng2023yarn},
xPos \citep{sun2023xpos}, and PI \citep{chen2023positioninterp}
adapt RoPE for longer contexts through frequency scaling or interpolation.
KPA introduces Kronecker-Separable RoPE, which applies rotary embeddings
directly to the factorized representation by rotating only the
feature-level factor, preserving the standard relative-position property
without materializing full key tensors or requiring a decoupled
positional branch.

\subsection{Kronecker factorization in transformers.}
Kronecker decomposition has been applied to transformers in several distinct
settings, all targeting static weight compression rather than inference-time
KV cache reduction. KroneckerBERT~\citep{tahaei2021kroneckerbert} compresses
BERT weight matrices via Kronecker decomposition with knowledge distillation.
KnGPT2~\citep{edalati2021knowgpt2} applies the same idea to GPT-2, replacing
linear layer weights with Kronecker-structured matrices and using intermediate
layer distillation to recover accuracy. KronA~\citep{edalati2022krona}
uses Kronecker products as a drop-in replacement for LoRA adapters in
parameter-efficient fine-tuning, achieving a better expressivity-to-parameter
trade-off than low-rank decomposition. HOT~\citep{omranpour2024higher}
decomposes the attention score matrix as a sum of Kronecker products to
reduce training-time computation for multi-dimensional encoder inputs.
KPA differs from all of these: it factorizes dynamic, per-token key and
value activations rather than static weight matrices or score matrices,
targeting inference-time memory reduction in autoregressive decoding,
a setting none of the above address.

\end{document}